\documentclass[conference]{IEEEconf}
\IEEEoverridecommandlockouts

\usepackage{cite}

\usepackage{amsmath,amssymb,amsfonts}
\usepackage{amsthm}
\usepackage{algorithmic}
\usepackage{graphicx}
\usepackage{textcomp}
\usepackage{xcolor}
\usepackage{algorithm}
\usepackage{subcaption}
\usepackage{url}
\usepackage{comment}
\newtheorem{lemma}{Lemma}
\newtheorem{problem}{Problem}

\def\BibTeX{{\rm B\kern-.05em{\sc i\kern-.025em b}\kern-.08em
    T\kern-.1667em\lower.7ex\hbox{E}\kern-.125emX}}
\begin{document}

\title{Reachability-Aware Time Scaling for 
Path Tracking
\thanks{This work is supported by ARO STIR Grant W911NF-25-1-0258}
\thanks{H. Gholampour and L. E. Beaver are with the Department of Mechanical and Aerospace Engineering and the Institute of Autonomous and Connected Systems,
Old Dominion University, Norfolk, VA 23529, USA (e-mails: \{mghol004,lbeaver\}@odu.edu).}
}

\author{Hossein Gholampour, \IEEEmembership{Graduate Student Member, IEEE},
Logan E. Beaver
\IEEEmembership{Member, IEEE}
}

\maketitle

\begin{abstract}
This paper studies tracking of collision-free waypoint paths produced by an offline planner for a planar double-integrator system with bounded speed and acceleration. Because sampling-based planners must route around obstacles, the resulting waypoint paths can contain sharp turns and high-curvature regions, so one-step reachability under acceleration limits becomes critical even when the path geometry is collision-free. We build on a pure-pursuit-style, reachability-guided quadratic-program (QP) tracker with a one-step acceleration margin. Offline, we evaluate this margin along a spline fitted to the waypoint path and update a scalar speed-scaling profile so that the required one-step acceleration remains below the available bound. Online, the same look-ahead tracking structure is used to track the scaled reference.
\end{abstract}


\section{Introduction}

Sampling-based motion planning is a common approach for generating collision-free paths in cluttered workspaces. Rapidly-exploring Random Trees (RRT) and their extensions build trees in the configuration space and can 
be adapted to high-dimensional systems \cite{lavalle2006planning,lavalle2001rrt}. RRT* adds an optimality mechanism that rewires nearby nodes and can converge to short paths \cite{karaman2011sampling}, and several variants improve its efficiency or bias the sampling, such as Informed RRT* \cite{gammell2014informed}, RRT*-Connect \cite{klemm2015rrt}, and RRTX \cite{otte2016rrtx}. In many implementations, these planners are used geometrically: they return a collision-free waypoint path, and a separate tracking controller is then responsible for tracking that path under speed and acceleration limits \cite{pham2017admissible,verginis2022kdf}. This separation is convenient, but the planner output may not be dynamically trackable under the true actuator bounds, especially near sharp turns \cite{orthey2023sampling,turkkol2024smooth}.

Given a collision-free geometric path, a common approach is to compute a dynamically feasible timing along the fixed path. In time-optimal path parameterization, reachability analysis can be used to construct feasible timing profiles under actuator limits \cite{pham2018new}. These methods focus on timing feasibility (often near time-optimal), but they typically do not address how a lightweight look-ahead tracker behaves on planner-generated paths under large execution disturbances, or how to rejoin the same path after freeze--resume events without replanning.

On the tracking side, a separate controller usually follows the planned path while enforcing dynamic limits such as speed and acceleration. Many systems use relatively simple position or velocity feedback laws, pure-pursuit-style tracking, or model predictive control 
and QP-based schemes to handle constraints explicitly \cite{jiang2025enhanced,nascimento2018nonholonomic,li2021model}. 
Pure pursuit remains popular in practice due to its simple geometry and implementation cost \cite{sulaiman2022implementation}. For constraint-enforced path following and recovery under disturbances, related work uses control barrier function (CBF) or reachability-based certificates to maintain safety/constraint satisfaction during tracking or return, e.g., robust CBF path following \cite{zheng2023constrained} and safe-return FaSTrack with robust control Lyapunov-value functions \cite{gong2024safe}. In our previous work \cite{gholampour2025trajectory}, we proposed a reachability-guided QP controller based on a one-step acceleration margin, and we showed that this margin can be used to verify whether a given 
trajectory is feasible within prescribed error bounds. 
However, that approach assumed a smooth, preplanned path with fixed timing and did not study paths produced by geometric planners (e.g., RRT, PRM) with many sharp corners and local curvature changes.

In this paper, we focus on the gap between planning and tracking. We start from a geometric waypoint path produced by a planner, such as RRT*, and we do not modify the planner itself. We fit a $C^2$ spline to the waypoints using a nominal speed and evaluate a one-step acceleration margin along the resulting path at sampled look-ahead targets.
When the margin indicates one-step infeasibility relative to the available acceleration bound, we apply time scaling to locally reduce the reference speed. 
The timing update is local: speed is reduced only where the bound is violated, while the rest of the path remains close to nominal speed. During execution, we use the same look-ahead tracking structure to follow the time-scaled reference.
We reuse the margin as a feasibility check for the look-ahead target, and rejoin the same path after interruptions or disturbances without replanning. Overall, this provides a reachability-aware offline time-scaling method that converts a geometric waypoint path into a timed reference that is consistent with one-step acceleration limits, with the same reachability-guided tracker to track the verified reference. 

The remainder of the paper is organized as follows: Section \ref{sec:problem} states the model and timing problem. Section \ref{sec:method} presents the controller and offline time-scaling design. Section \ref{sec:setup_results} presents simulation results. Section \ref{sec:conclusion} concludes and discusses limitations and future work.

\section{Problem Formulation} \label{sec:problem}

We consider a robotic system that tracks a planar trajectory in the presence of static obstacles.
The 
position and velocity are
$\mathbf{p}(t), \mathbf{v}(t) \in \mathbb{R}^2$, and the commanded acceleration is
$\mathbf{u}(t) \in \mathbb{R}^2$.
We model the tracked output as a double integrator with bounded disturbances:
\begin{equation}
    \dot{\mathbf{p}}(t) = \mathbf{v}(t) + \mathbf{n}_p(t), 
    \quad 
    \dot{\mathbf{v}}(t) = \mathbf{u}(t) + \mathbf{n}_v(t),
    \label{eq:continuous_model}
\end{equation}
where $\mathbf{n}_p(t)$ and $\mathbf{n}_v(t)$ are perturbations that satisfy
\begin{equation}
    \|\mathbf{n}_p(t)\| \leq \varepsilon_p,
    \quad
    \|\mathbf{n}_v(t)\| \leq \varepsilon_v.
    \label{eq:noise_bounds}
\end{equation}
This model applies to manipulators in task space, mobile robots, and other systems whose
tracked output can be feedback-linearized to a double integrator \cite{gholampour2025trajectory}.

Hardware and safety limits impose magnitude bounds on speed and acceleration,
\begin{equation}
    \|\mathbf{v}_k\| \leq v_{\max},
    \qquad
    \|\mathbf{u}_k\| \leq a_{\max},
    \label{eq:constraints}
\end{equation}
We implement the tracking controller on a digital controller with sampling period $t_s$, with sample times $t_k = k t_s$, for $k = 1, 2, \dots$ time steps. A forward-Euler discretization gives
\begin{align}
    \mathbf{v}_{k+1} &= \mathbf{v}_k + (\mathbf{u}_k + \mathbf{n}_{v,k}) t_s, 
    \label{eq:discrete_v} \\
    \mathbf{p}_{k+1} &= \mathbf{p}_k + (\mathbf{v}_k + \mathbf{n}_{p,k}) t_s 
    + \frac{1}{2} (\mathbf{u}_k + \mathbf{n}_{v,k}) t_s^2.
    \label{eq:discrete_p}
\end{align}

These sampled equations define the one-step model used in the reachability margin in Section~\ref{subsec:margin}.
Let the position workspace be a compact set $\mathcal{X}\subset\mathbb{R}^2$.
It contains a finite set of static obstacles
\begin{equation}
    \mathcal{X}_{\text{obs}}^{(i)} \subset \mathcal{X}, 
    \quad i = 1,\dots,N_{\text{obs}}.
\end{equation}
The total obstacle region is
\begin{equation}
    \mathcal{X}_{\text{obs}} = \bigcup_{i=1}^{N_{\text{obs}}} \mathcal{X}_{\text{obs}}^{(i)},
\end{equation}
and the free space is
\begin{equation}
    \mathcal{X}_{\text{free}} = \mathcal{X} \setminus \mathcal{X}_{\text{obs}}.
\end{equation}
We do not assume any specific shape for $\mathcal{X}_{\text{obs}}^{(i)}$; we only require that
collision checking can decide membership in $\mathcal{X}_{\text{obs}}$.

Given start and goal positions
$\mathbf{p}_{\text{start}}, \mathbf{p}_{\text{goal}} \in \mathcal{X}_{\text{free}}$,
we assume that an initial feasible path is available in the form of a collision-free waypoint sequence
\begin{equation} \label{eq:waypoint_list}
    \mathcal{W} = \{\mathbf{w}_0, \mathbf{w}_1, \ldots, \mathbf{w}_{N_{\text{wp}}}\},
    \quad 
    \mathbf{w}_i \in \mathcal{X}_{\text{free}},
\end{equation}
with $\mathbf{w}_0 = \mathbf{p}_{\text{start}}$ and
$\mathbf{w}_{N_{\text{wp}}} = \mathbf{p}_{\text{goal}}$.
These waypoints define a collision-free polyline but may contain sharp corners, and they do not include timing information.

To obtain a smooth reference, we fit a $C^2$ spline through the waypoints.
We assign spline knot parameters $\tau_i$ from the cumulative Euclidean distance along the waypoint polyline and then scale them by a constant so that $\tau$ is a nominal time-like parameter over $[0,\tau_{\text{end}}]$. With this choice, $\mathbf{p}'_{\text{ref}}(\tau)$ represents a nominal velocity field along the spline. This yields a reference position trajectory,
\begin{equation}
    \mathbf{p}_{\text{ref}}(\tau) : [0, \tau_{\text{end}}] \to \mathcal{X}_{\text{free}},
    \label{eq:pref_def}
\end{equation}
with derivative $\mathbf{p}'_{\text{ref}}(\tau)$. We treat $\tau$ as a scalar path parameter that measures progress along the reference spline, and physical time enters through a timing law $\tau(t)$ that maps $t \ge 0$ into $[0,\tau_{\text{end}}]$. We consider a baseline timing $\dot{\tau}(t)=1$ (so $\tau(t)=t$ during motion), and apply timing modifications through a time-scaling function $\alpha(t)=\dot{\tau}(t)$ relative to this baseline.

The state of the system at sample $k$ is
$(\mathbf{p}_k, \mathbf{v}_k)$, and the reference is the spline
$\mathbf{p}_{\text{ref}}(\tau)$ in \eqref{eq:pref_def}.
The tracking goal is to move from $\mathbf{p}_{\text{start}}$ to $\mathbf{p}_{\text{goal}}$ along the reference spline, 
while respecting the limits in \eqref{eq:constraints} for all $k$.

We also consider freeze--resume events, motivated by collaborative robotics.
A freeze event is triggered when a safety condition is met (for example, a force threshold, proximity sensor, or emergency stop). During the freeze interval $[t_f,t_r]$ the end-effector is required to remain at the frozen position with zero velocity,
\begin{equation}
    \mathbf{p}(t) = \mathbf{p}(t_f), \quad
    \mathbf{v}(t) = \mathbf{0}, \quad
    \mathbf{u}(t) = \mathbf{0}, \quad
    t \in [t_f, t_r].
\end{equation}
After $t = t_r$ the controller must rejoin the same reference path $\mathbf{p}_{\text{ref}}(\tau)$ without replanning, while satisfying the sampled bounds in \eqref{eq:constraints}.

We make the following assumptions:

\begin{enumerate}
    \item \textbf{Real-time feasibility:}
    Computation and communication delays are small relative to $t_s$, so the selected control $\mathbf{u}_k$ is applied within the sample. 
    \item \textbf{Planner output (polyline):}
    A geometric planner provides a waypoint sequence from \eqref{eq:waypoint_list} such that the straight-line segments between consecutive waypoints lie in $\mathcal{X}_{\text{free}}$.
    \item \textbf{Spline fit (final path):}
    The fitted $C^2$ spline $\mathbf{p}_{\text{ref}}(\tau)$ constructed from $\mathcal{W}$ is regular and satisfies $\mathbf{p}_{\text{ref}}(\tau) \in \mathcal{X}_{\text{free}}$ for all $\tau \in [0,\tau_{\text{end}}]$.

\end{enumerate}

The real-time feasibility assumption is typically satisfied for an onboard digital controller when computation and communication delays are small relative to $t_s$. The spline collision-free assumption is more restrictive, since smoothing a collision-free polyline can introduce collisions near obstacles. In practice, this can be addressed by inflating obstacles during planning, by checking the fitted spline for collisions and refining waypoints locally, or by using a smoothing method with collision constraints. 
Generating and certifying smooth collision-free paths is beyond the scope of this paper.
Under these assumptions, we address the following problem.

\begin{problem}
Given the disturbance bounds \eqref{eq:noise_bounds}, the sampled dynamics \eqref{eq:discrete_v}--\eqref{eq:discrete_p}, the sampled speed and acceleration bounds \eqref{eq:constraints}, and a collision-free spline reference $\mathbf{p}_{\text{ref}}(\tau)$ in \eqref{eq:pref_def}, design (i) an offline time-scaling law along the path, and (ii) an online tracking controller such that:
\begin{itemize}
    \item the motion satisfies the sampled speed and acceleration bounds at each control update,
    \item the system can rejoin the same path after a freeze interval without replanning,
    \item and the tracking objective is maintained on planner-generated paths with sharp turns.
\end{itemize}
\end{problem}

\section{Controller and Time-Scaling Design} \label{sec:method}
We build on our reachability-guided QP tracking controller from \cite{gholampour2025trajectory}. In that work, we developed a one-step reachability margin 
to detect when the look-ahead target was too aggressive and to support freeze--resume recovery under bounded disturbances. Here, we use the same feasibility check in two places: offline, we evaluate the margin along the spline and construct a local time-scaling profile. 
Online, 
we use the controller to track the reference spline while enforcing the sampled speed and acceleration bounds.
Under Assumption 3, obstacle avoidance is handled at the planning stage through the collision-free reference.
While we do not include an explicit obstacle-avoidance layer here, such constraints could be incorporated with CBF-based collision-avoidance methods \cite{zheng2023constrained}.

\subsection{Look-Ahead and One-Step Reachability Margin}
\label{subsec:margin}
Offline, we compute a scalar time-scaling profile $\alpha(\tau)\in(0,1]$ on the $\tau$ grid, where $\alpha(\tau_k) < 1$ effectively reduces the speed of the reference profile at time step $k$.
We achieve this by setting,
\begin{equation}
    \dot{\tau}_k = \alpha(\tau_k),
    \label{eq: tau_dot}
\end{equation}
which determines the path progress,
\begin{equation}
    \tau_{k+1} = \min\{\tau_{\text{end}},\; \tau_k + \dot{\tau}_k t_s\},
    \label{eq:tau_update}
\end{equation}
where $\tau_{\text{end}}$ clamps the reference at the arrival time.
%
%
We denote the speed profile of the trajectory by
\begin{equation}
v_{\text{path}}(\tau) := \alpha(\tau)\,\|\mathbf{p}'_{\text{ref}}(\tau)\|,
\label{eq:vpath_def}
\end{equation}
which follows from the chain rule.
This is how time-scaling profile $\alpha(\tau)$ affects the reference speed.

At each sample $k$ we project the current position $\mathbf{p}_k$ onto the reference path, 
\begin{equation}
    \tau_c = \arg\min_{\tau \in [0,\tau_{\text{end}}]}
    \big\|\mathbf{p}_{\text{ref}}(\tau) - \mathbf{p}_k\big\|.
    \label{eq:closest_point}
\end{equation}
We use this to determine the look-ahead speed,
\begin{equation}
v_{\text{LA},k} := \min\{v_{\max},\, \dot{\tau}_k\|\mathbf{p}'_{\text{ref}}(\tau_c)\|\},
\label{eq:vla_def}
\end{equation}
which we use to select the look-ahead distance,
\begin{equation}
s_{\text{LA},k} = v_{\text{LA},k}\, t_s .
\label{eq:sla_rule}
\end{equation}
This choice makes the look-ahead distance proportional to the local reference speed. 


Let $S(\tau)$ denote the arc-length function of the spline,
\begin{equation}
    S(\tau) = \int_{0}^{\tau} \left\|\mathbf{p}'_{\text{ref}}(\xi)\right\| d\xi,
    \label{eq:arc_length_def}
\end{equation}
where $\xi$ is an integration variable.
By Assumption 3, $S(\tau)$ is strictly increasing, so the inverse $S^{-1}(\cdot)$ is well-defined. We define the look-ahead parameter by advancing the arc length from the closest point:
\begin{equation}
    \tau_{\text{LA},k} =
    \min\!\left\{
    \tau_{\text{end}},\;
    S^{-1}\!\big(S(\tau_c)+s_{\text{LA},k}\big)
    \right\}.
    \label{eq:tau_la}
\end{equation}
Given $s_{\text{LA},k}$, this construction defines the look-ahead parameter geometrically through arc length. 
Figure \ref{fig:lookahead_geom} illustrates this sampled closest-point/look-ahead construction, where $\mathbf{p}_c=\mathbf{p}_{\text{ref}}(\tau_c)$ denotes the closest sampled reference point and $\mathbf{p}_{\text{LA},k}=\mathbf{p}_{\text{ref}}(\tau_{\text{LA},k})$ denotes the look-ahead target. The black dots indicate sampled reference points, and the highlighted arc represents the accumulated arc length used to advance from $\tau_c$ to $\tau_{\text{LA},k}$. 
%
The resulting look-ahead target is,
\begin{equation}
    \mathbf{p}_{\text{LA},k} = \mathbf{p}_{\text{ref}}(\tau_{\text{LA},k}).
    \label{eq:pla}
\end{equation}
 
\begin{figure}[t]
    \centering
    \includegraphics[width=1 \linewidth]{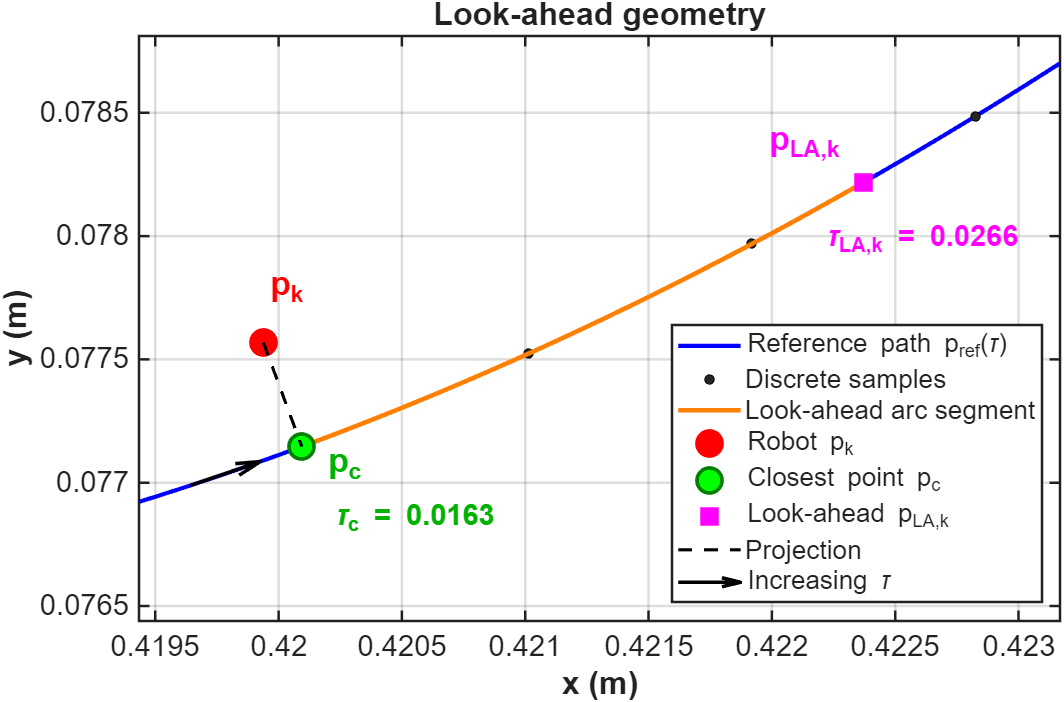}
    \caption{Look-ahead geometry on the reference path. The closest point $p_c=p_{\mathrm{ref}}(\tau_c)$ is obtained by projecting $p_k$ onto the path, and the look-ahead target $p_{\mathrm{LA},k}=p_{\mathrm{ref}}(\tau_{\mathrm{LA},k})$ is selected by advancing arc length $s_{\mathrm{LA},k}$ from $\tau_c$.}
    \label{fig:lookahead_geom}
\end{figure}


In the remainder of this sub-section, we briefly introduce our tracking controller \cite{gholampour2025trajectory}.
For a given state $(\mathbf{p}_k,\mathbf{v}_k)$ and target $\mathbf{p}_{\text{LA},k}$, we compute the nominal one-step acceleration that would place the system at $\mathbf{p}_{\text{LA},k}$ after one sample:
\begin{equation}
    \mathbf{u}_{\text{req},k}
    = \frac{2}{t_s^2} \left( \mathbf{p}_{\text{LA},k} - \mathbf{p}_k - t_s \mathbf{v}_k \right).
    \label{eq:u_req_vec}
\end{equation}
For notational convenience, we define the required control effort magnitude in terms of the vector $\mathbf{r}_k$,
\begin{equation}
    \mathbf{r}_k = \mathbf{p}_{\text{LA},k} - \mathbf{p}_k - t_s \mathbf{v}_k,
    \quad
    u_{\text{req},k} = \frac{2}{t_s^2} \big\|\mathbf{r}_k\big\|.
    \label{eq:u_req_scalar}
\end{equation}

To account for bounded disturbances, we convert the bounds $\varepsilon_p,\varepsilon_v$ into an equivalent acceleration offset. The worst-case one-step position shift due to disturbance is
\begin{equation}
    \Delta p_{\max} = \varepsilon_p t_s + \tfrac{1}{2} \varepsilon_v t_s^2.
\end{equation}
The corresponding disturbance offset is
\begin{equation}
    \sigma = \frac{2 \Delta p_{\max}}{t_s^2}
    = \frac{2 \varepsilon_p}{t_s} + \varepsilon_v.
    \label{eq:sigma}
\end{equation}
We then define the one-step reachability margin
\begin{equation}
    \delta_k = u_{\text{req},k} - (a_{\max} - \sigma),
    \label{eq:delta}
\end{equation}
where $a_{\max}-\sigma$ is the available acceleration after accounting for bounded disturbances.
If $\delta_k \le 0$, the look-ahead target is one-step feasible under the available acceleration bound $a_{\max}-\sigma$. If $\delta_k > 0$, the target is not guaranteed to be reachable in one step under the acceleration limit and robustness margin.

We use the same margin to select the control action during online tracking. Consider the one-step position error,
\begin{equation}
    \mathbf{e}_{p,k} =
    \mathbf{p}_{\mathrm{LA},k} - \left(\mathbf{p}_k + t_s\mathbf{v}_k\right).
\end{equation}
Under our proposed time scaling, the reference velocity at the look-ahead target is
\begin{equation} \label{eq:vref}
    \mathbf{v}_{\text{ref},k} = \dot{\tau}_k\,\mathbf{p}'_{\mathrm{ref}}(\tau_{\mathrm{LA},k}),
\end{equation}
so that
\begin{equation}
    \mathbf{e}_{v,k} = \mathbf{v}_{\text{ref},k} - \mathbf{v}_k.
\end{equation}
If $\delta_k \le 0$, we have sufficient control authority to bring the position error to zero, i.e.,
\begin{equation}
    \mathbf{u}_k^\star = \frac{2}{t_s^2}\mathbf{e}_{p,k}.
    \label{eq:u_star_k}
\end{equation}
If $\delta_k > 0$, we minimize a combination of position and velocity error to avoid overshooting the reference \cite{gholampour2025trajectory},
\begin{equation}
    J = \frac{1}{2}||\mathbf{e}_{p,k}||^2 + \frac{C_k}{2} ||\mathbf{e}_{v,k}||^2
\end{equation}
where $C_k>0$ is an adaptive weight that we update online.
The optimal control input is,
\begin{equation}
    \mathbf{u}_k^\star =
    \frac{\mathbf{e}_{p,k} + \dfrac{2C_k}{t_s}\mathbf{e}_{v,k}}
    {\dfrac{1}{2}t_s^2 + 2C_k}.
\end{equation}
We then clip $\mathbf{u}_k^\star$ to satisfy the sampled speed and acceleration bounds.

The next lemma identifies a local mechanism that can lead to transient overshoot and oscillatory corrections in the one-step position-matching law \eqref{eq:u_star_k}.
\begin{lemma}
\label{lem:chord_tangent}

Consider the disturbance-free sampled model \eqref{eq:discrete_v}--\eqref{eq:discrete_p}, and suppose the unclipped control law \eqref{eq:u_star_k} is applied. Then the next-step velocity is
\begin{equation}
    \mathbf{v}_{k+1}  = -\mathbf{v}_k + \frac{2}{t_s}
    \left(
    \mathbf{p}_{\mathrm{LA},k}  -  \mathbf{p}_k
    \right).
    \label{eq:vf_chord}
\end{equation}
Hence exact one-step position matching is compatible with the desired reference velocity $\mathbf{v}_{\mathrm{ref},k}$ if and only if
\begin{equation}
    \mathbf{m}_k := 2\left(\mathbf{p}_{\mathrm{LA},k}-\mathbf{p}_k\right)-t_s\left(\mathbf{v}_k+\mathbf{v}_{\mathrm{ref},k}\right)=\mathbf{0}
    \label{eq:chord_tangent_mismatch}
\end{equation}
\end{lemma}

\begin{proof}
Substituting \eqref{eq:u_star_k} into the disturbance-free velocity update \eqref{eq:discrete_v} gives
\[
\mathbf{v}_{k+1}=\mathbf{v}_k + \mathbf{u}_k^\star t_s=\mathbf{v}_k + \frac{2}{t_s^2}\mathbf{e}_{p,k}\, t_s.
\]
Using
\[
\mathbf{e}_{p,k}=\mathbf{p}_{\mathrm{LA},k}-\left(\mathbf{p}_k+t_s\mathbf{v}_k\right),
\]
we obtain
\begin{equation*}
\begin{aligned}
\mathbf{v}_{k+1}
&=
\mathbf{v}_k
+
\frac{2}{t_s}
\left(
\mathbf{p}_{\mathrm{LA},k}
-
\mathbf{p}_k
-
t_s\mathbf{v}_k
\right) \\
&=
-\mathbf{v}_k
+
\frac{2}{t_s}
\left(
\mathbf{p}_{\mathrm{LA},k}
-
\mathbf{p}_k
\right).
\end{aligned}
\end{equation*}
which proves \eqref{eq:vf_chord}. Requiring
\(
\mathbf{v}_{k+1}=\mathbf{v}_{\mathrm{ref},k}
\)
and rearranging yields \eqref{eq:chord_tangent_mismatch}.
\end{proof}

Lemma \ref{lem:chord_tangent} shows that exact one-step position matching determines a terminal velocity through the chord from $\mathbf{p}_k$ to $\mathbf{p}_{\mathrm{LA},k}$, whereas the desired reference velocity is tangent-based through \eqref{eq:vref}. On a curved path, or when the robot is off the path, these two directions generally differ. The mismatch vector $\mathbf{m}_k$ therefore provides a local measure of chord--tangent inconsistency. Repeated one-step corrections against this mismatch can produce transient overshoot and oscillatory corrections, which provide a local explanation for the chattering observed in simulation.

\subsection{Offline Reachability-Aware Time Scaling}
\label{subsec:offline_scaling}

We use the margin $\delta_k$ to adapt the time-scaling profile along the path before execution. We keep the geometric path $\mathbf{p}_{\text{ref}}(\tau)$ fixed and reduce the desired speed only where the one-step look-ahead reachability check indicates excessive required acceleration.
Offline, we compute the scalar time-scaling profile $\alpha(\tau)\in(0,1]$ at discrete sample points $\tau_i$, as shown in Fig. \ref{fig:lookahead_geom}.
This affects the reference velocity \eqref{eq:vref} 
which determines the look-ahead distance via 
\eqref{eq:sla_rule}.
We define the available acceleration by
\begin{equation}
    a_{\text{avail}} = a_{\max} - \sigma.
    \label{eq:a_avail}
\end{equation}

At each sampled path location $\tau_i$, let $\alpha_i := \alpha(\tau_i)$ and evaluate the one-step required acceleration using the reference state
\begin{equation}
    \mathbf{p}_i = \mathbf{p}_{\text{ref}}(\tau_i),
    \qquad
    \mathbf{v}_i = \alpha_i \mathbf{p}'_{\text{ref}}(\tau_i).
\end{equation}
Let
\begin{equation}
    \Delta \tau_i := \tau_{\text{LA},i} - \tau_i .
\end{equation}
If $\alpha$ is locally constant over one control step, then the arc-length look-ahead construction yields the local approximation
\begin{equation}
    \Delta \tau_i \approx \alpha_i t_s .
\end{equation}
Let
\begin{equation}
    \mathbf{p}_{\text{LA},i} := \mathbf{p}_{\text{ref}}(\tau_{\text{LA},i})
    = \mathbf{p}_{\text{ref}}(\tau_i + \Delta \tau_i).
\end{equation}
Using a second-order Taylor expansion of $\mathbf{p}_{\text{ref}}$ about $\tau_i$,
\begin{equation}
    \mathbf{p}_{\text{LA},i} = \mathbf{p}_i + \mathbf{p}'_{\text{ref}}(\tau_i)\Delta \tau_i + \frac{1}{2}\mathbf{p}''_{\text{ref}}(\tau_i)\Delta \tau_i^2 + \cdots,
\end{equation}
and using the one-step mismatch in \eqref{eq:u_req_scalar}, evaluated at the offline sample $i$, the first-order term cancels, so the leading term is quadratic in $\alpha_i$:
\begin{equation}
    \mathbf{r}_i \approx \frac{1}{2}\mathbf{p}''_{\text{ref}}(\tau_i)\alpha_i^2 t_s^2 .
\end{equation}
Therefore,
\begin{equation}
    |u_{\text{req},i}|  =  \frac{2}{t_s^2}\|\mathbf{r}_i\|
    \approx  \alpha_i^2 \|\mathbf{p}''_{\text{ref}}(\tau_i)\|,
\end{equation}
so if the local scale changes from $\alpha_i$ to $\alpha_i^{\text{new}}$, we obtain the local second-order approximation
\begin{equation}
    |u_{\text{req},i}^{\text{new}}|
    \approx
    \left(\frac{\alpha_i^{\text{new}}}{\alpha_i}\right)^2
    |u_{\text{req},i}|.
    \label{eq:ur_quadratic_scaling}
\end{equation}
This approximation motivates the update
\begin{equation}
    \alpha_i^{\text{new}} = \alpha_i \sqrt{\frac{a_{\text{avail}}}{|u_{\text{req},i}|}},
    \label{eq:alpha_update}
\end{equation}
which we apply  whenever $|u_{\text{req},i}|>a_{\text{avail}}$.

To avoid isolated pointwise updates on the sampled path, we apply the same reduction over a local neighborhood of the recorded look-ahead location. Let
\begin{equation}
    s_i^{\mathrm{LA}} := S(\tau_{\mathrm{LA},i}),
\end{equation}
and define the neighborhood
\begin{equation}
    \mathcal{N}_i := \left\{ j : \left| S(\tau_j) - s_i^{\mathrm{LA}} \right| \leq w_{\mathrm{n}} \right\},
    \label{eq:neighborhood_def}
\end{equation}
where $w_{\mathrm{n}}$ is a fixed arc-length half-window. For each violating sample $i$, we update all grid points $j \in \mathcal{N}_i$ by
\begin{equation}
    \alpha_j \leftarrow \min\!\left( \alpha_j,\; \alpha_i^{\text{new}} \right).
    \label{eq:neighborhood_update}
\end{equation}
If multiple violating samples produce overlapping neighborhoods, we keep the smallest assigned value.

Algorithm \ref{alg:practical_ts} summarizes the offline time-scaling procedure used in the simulations. Lines 1--3 initialize the nominal profile and record the one-step feasibility data from the nominal run. Lines 4--6 apply the local update \eqref{eq:alpha_update} over the neighborhood \eqref{eq:neighborhood_def} at samples with $\delta_i>0$. Line 7 smooths the resulting discrete profile on the sampled path grid, and Lines 8--9 construct the final time-scaled reference. This construction preserves the nominal speed on path segments that already satisfy the one-step bound and reduces the speed only where the nominal run indicates infeasibility.

\begin{algorithm}[ht]
\caption{Offline time-scaling procedure}
\label{alg:practical_ts}
\begin{algorithmic}[1]
\REQUIRE Nominal reference path $\mathbf{p}_{\text{ref}}(\tau)$, 
limits $a_{\max}$, robustness margin $\sigma$
\STATE Run the nominal tracker with $\alpha(\tau)= 1$
\STATE Record $\delta_i$, $u_{\text{req},i}$, and the corresponding look-ahead location for each discretization point $i$
\STATE Initialize $\alpha_i \leftarrow 1$ at each grid point
\FOR{each recorded sample $i$ with $\delta_i > 0$}
    \STATE For all $j \in \mathcal{N}_i$, update 
    $\alpha_j \leftarrow \min\!\left(\alpha_j,\sqrt{\dfrac{a_{\text{avail}}}{u_{\text{req},i}}}\right)$
\ENDFOR
\STATE Smooth the discrete profile $\{\alpha_i\}$ on the sampled path grid using a centered moving-average window of length $N_{\text{smooth}}$
\STATE Construct the scaled reference velocity field $\mathbf{v}_{\text{ref}}(\tau_i)=\alpha_i\,\mathbf{p}'_{\text{ref}}(\tau_i)$
\RETURN Scaled reference field $(\mathbf{p}_{\text{ref}}(\tau_i),\mathbf{v}_{\text{ref}}(\tau_i))$
\end{algorithmic}
\end{algorithm}

The offline-scaled results reported in Section \ref{sec:setup_results} use Algorithm \ref{alg:practical_ts}. 

\section{Simulation Setup and Results} \label{sec:setup_results}

\subsection{Simulation Setup} \label{subsec:setup}
We simulate the planar double-integrator model in a square workspace $\mathcal{X}=[0,L_x]\times[0,L_y]$. The start and goal positions $\mathbf{p}_{\text{start}}$ and $\mathbf{p}_{\text{goal}}$ are placed in opposite corners so that any feasible path must pass through the obstacle region. The workspace contains $N_{\text{obs}}$ circular obstacles with randomly sampled radii $r_i\in[r_{\min}, r_{\max}]$ and centers drawn uniformly in $\mathcal{X}$, subject to a minimum clearance from the start and goal. 
Obstacles that violate the start/goal clearance are rejected.
The simulation parameters are given in Table \ref{tab:sim_params}.

We first generate an initial collision-free waypoint path with a geometric sampling-based planner (RRT*) and then backtrack from the goal node to the root to obtain a waypoint sequence. We remove duplicate waypoints and fit a $C^2$ cubic spline $\mathbf{p}_{\text{ref}}(\tau)$ through the remaining points. For the offline time-scaling and reachability checks, this spline and its derivatives are sampled on a fine grid in $\tau$. We intentionally choose an aggressive nominal timing to parameterize the spline. This produces a baseline reference that violates the state and/or control constraints, 
so that the effect of the offline time scaling can be evaluated. 
Fig. \ref{fig:setup_overview} shows one representative trial in the workspace, including the obstacles, the RRT* waypoint polyline, the fitted spline, and the executed trajectory. The red portions of the executed trajectory indicate path segments where the offline time-scaling profile is active, i.e., where $\alpha(\tau)<1$.

\begin{figure}[ht]
  \centering
  \includegraphics[width=1\linewidth]{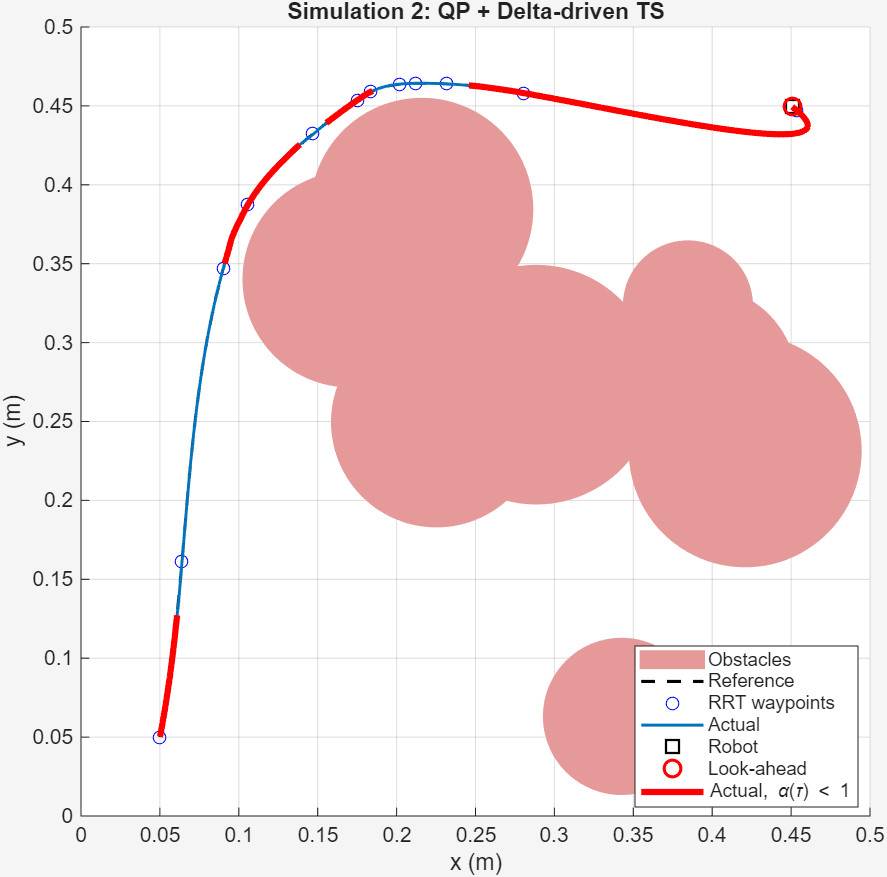}
  \caption{Representative trial in the workspace with active time-scaled trajectory segments highlighted in red.}
  \label{fig:setup_overview}
\end{figure}

To model freeze--resume behavior, each simulation includes a single freeze interval. During the freeze interval, we set $\mathbf{v}(t)=\mathbf{0}$ and $\mathbf{u}(t)=\mathbf{0}$ and hold the position constant in simulation. 
The path parameter $\tau$ continues to evolve during the freeze, but our geometric
look-ahead target 
prevents the target from moving far ahead of the robot. After the resume time, the tracker continues along the same reference path without replanning. The simulation code is available online
\footnote{\url{https://github.com/Hossein-ghd/Reachability_Aware_Time-Scaling}}.
%
%
We compare two cases:
\begin{itemize}
    \item \textbf{Nominal timing:} The closed-form online tracker runs with $\alpha(\tau)= 1$, i.e., without offline time scaling.
    \item \textbf{Offline time scaling:} The same online tracker runs on the time-scaled reference produced by Algorithm \ref{alg:practical_ts}. 
\end{itemize}

Both cases use the same geometric waypoint path, the same look-ahead rule, and the same freeze--resume protocol. Therefore, any difference between the two cases is attributable to the offline timing modification rather than to a change in the online controller structure.

\begin{table}[t]
\caption{Simulation and planning parameters.}
\label{tab:sim_params}
\centering
\begin{tabular}{l c}
\hline
Parameter & Value \\
\hline
$t_s$ (control period) & $0.0125$ s $(80Hz)$ \\
Nominal path horizon & $2$ s \\
$v_{\max}$ & $1.0$ m/s \\
$a_{\max}$ & $2.5$ m/s$^2$ \\
$L_x, L_y$ & $0.5$ m, $0.5$ m \\
$\mathbf{p}_{\text{start}}$ & $[0.05,\;0.05]^T$ m \\
$\mathbf{p}_{\text{goal}}$ & $[0.45,\;0.45]^T$ m \\
$\alpha_{\min}$ & $0.1$ \\
Local slowdown half-window & $0.02$ m \\
$N_{\text{smooth}}$ (smoothing window length) & $201$ points \\
Reference grid size $M$ & $150000$ \\
\hline
\end{tabular}
\end{table}

\subsection{Results} \label{subsec:results}
We ran $50$ randomized trials using the setup from Section \ref{subsec:setup}. For each trial, we compute the offline time-scaling profile from the nominal run and then evaluate the scaled run on the same path and freeze timing. A representative profile is shown in Fig. \ref{fig:alpha_profile}. The profile remains at $\alpha=1$ over substantial portions of the path and drops below $1$ only on selected intervals where the nominal run indicates one-step infeasibility. Thus, the timing modification is local rather than uniform. For the same representative trial, Fig. \ref{fig:setup_overview} shows that these active time-scaled intervals correspond to selected portions of the executed trajectory rather than the full path. Table \ref{tab:alpha_stats} summarizes the profile statistics over all $50$ trials.

\begin{figure}[ht]
    \centering
    \includegraphics[width=1\linewidth]{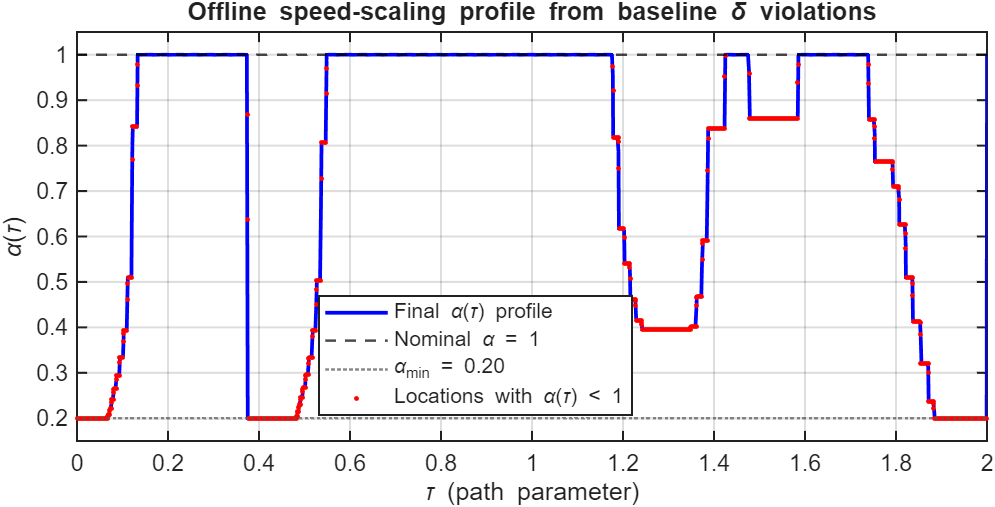}
    \caption{Representative offline time-scaling profile $\alpha(\tau)$.}
    \label{fig:alpha_profile}
\end{figure}

\begin{table}[ht]
\caption{Offline time-scaling profile statistics over 50 trials.}
\label{tab:alpha_stats}
\centering
\begin{tabular}{lc}
\hline
Metric & Value \\
\hline
Minimum $\alpha$ & $0.163 \pm 0.053$ \\
Mean $\alpha$ & $0.846 \pm 0.036$ \\
Modified path points (\%) & $26.98 \pm 9.66$ \\
\hline
\end{tabular}
\end{table}

Although the offline profile is constructed from the one-step margin, the offline construction and the online $\delta_k$ are not the same computation. Offline, the method uses $\delta_k$ and $u_{\mathrm{req},k}$ recorded during a baseline run with $\alpha(\tau)=1$ and assigns slowdown near the corresponding recorded look-ahead locations. Online, $\delta_k$ is recomputed from the simulated state $(\mathbf{p}_k,\mathbf{v}_k)$ during execution of the scaled reference, with $\tau_c$ obtained from \eqref{eq:closest_point}. These quantities can differ due to state mismatch, constraint saturation, closest-point index changes, and post-freeze transients. Because $u_{\mathrm{req},k}$ in \eqref{eq:u_req_scalar} scales with $1/t_s^2$, even small mismatch can produce occasional positive $\delta_k$ spikes during execution. Therefore, a profile that improves the baseline margin signal does not imply $\delta_k \le 0$ at every online sample of the scaled run.

Table \ref{tab:results_summary} reports aggregate performance over the moving interval for the 50 randomized trials. Offline time scaling reduces the occurrence of $\delta_k>0$ and lowers the mean and maximum values of $\delta$. It also lowers the mean and maximum speed, which is expected since the method slows the reference locally. These changes are consistent with Fig. \ref{fig:alpha_profile}, where the scaling acts only over selected path segments rather than over the full trajectory.

\begin{table}[ht]
\caption{Summary over 50 randomized trials over the moving interval. Values are mean $\pm$ standard deviation across trials.}
\label{tab:results_summary}
\centering
\begin{tabular}{lcc}
\hline
Metric & QP & QP + Offline TS \\
\hline
Mean $\delta$ (m/s$^2$) 
& $9.924 \pm 9.916$ 
& $\mathbf{0.382 \pm 3.366}$ \\

5th percentile of $\delta$ (m/s$^2$) 
& $-1.849 \pm 0.363$ 
& $\mathbf{-1.890 \pm 0.272}$ \\

Samples with $\delta_k>0$ (\%) 
& $21.03 \pm 10.18$ 
& $\mathbf{8.96 \pm 4.27}$ \\

Max $\delta$ (m/s$^2$) 
& $130.24 \pm 89.42$ 
& $\mathbf{42.59 \pm 46.47}$ \\

Mean speed (m/s) 
& $0.288 \pm 0.054$ 
& $\mathbf{0.157 \pm 0.052}$ \\

Max speed (m/s) 
& $0.367 \pm 0.072$ 
& $\mathbf{0.344 \pm 0.050}$ \\
\hline
\end{tabular}
\end{table}

Figure \ref{fig:delta_time} plots $\delta_k$ for a representative trial. The shaded region denotes the freezing period, which is excluded from the reported moving-time metrics. After resume, both methods exhibit a transient increase in $\delta_k$, but the offline-scaled reference reduces the time spent in positive-margin regions during normal motion. Small oscillations can still appear in the speed and $\delta_k$ traces due to sampled closest-point/look-ahead updates and the chord--tangent mismatch described in Lemma \ref{lem:chord_tangent}.

\begin{figure}[ht]
    \centering
    \includegraphics[width=1\linewidth]{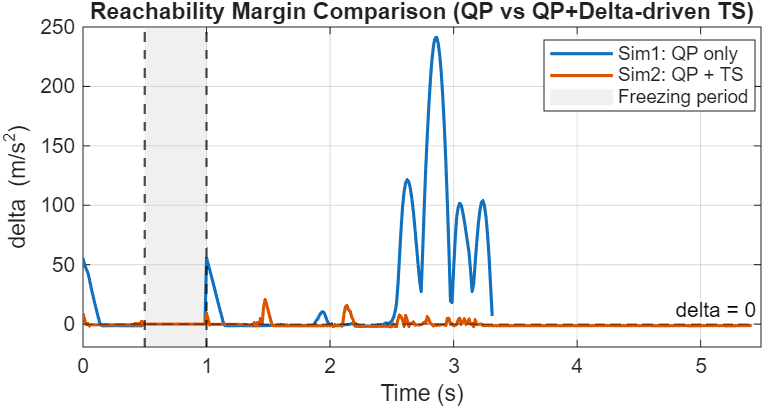}
    \caption{Representative trial: $\delta_k$ versus time with the freezing period indicated (QP vs. QP+TS).}
    \label{fig:delta_time}
\end{figure}

\begin{figure}[ht]
    \centering
    \includegraphics[width=1\linewidth]{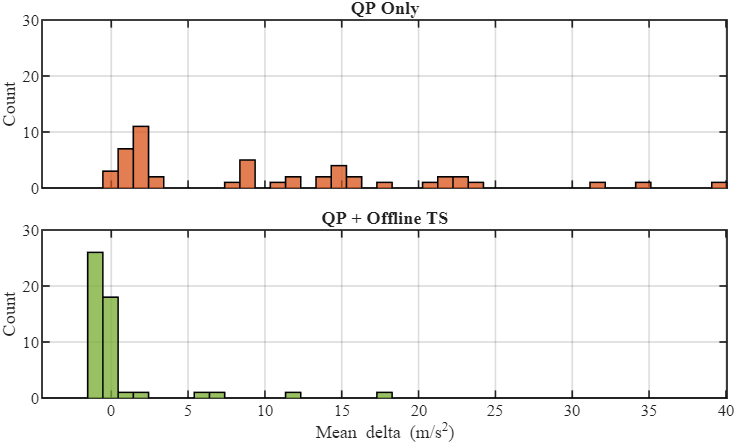}
    \caption{Histogram of per-trial mean $\delta$ over moving samples (50 trials).}
    \label{fig:hist_mean_delta}
\end{figure}

Figure \ref{fig:hist_mean_delta} shows the run-to-run distribution of the per-trial mean margin, where each bin count corresponds to one trial’s average of $\delta_k$ over the moving interval.
Compared to QP-only, QP+Offline TS shifts more trials toward negative mean margins and reduces the number of trials with mean $\delta$ close to zero, i.e., runs that operate nearer the one-step feasibility boundary.
 
Offline reachability-aware time scaling consistently reduces one-step infeasibility events while using the same online controller structure. The improvement is localized to selected path segments, as shown by the offline time-scaling profile in Fig. \ref{fig:alpha_profile} and the highlighted trajectory segments in Fig. \ref{fig:setup_overview}, rather than being applied uniformly along the full path.

\section{Conclusion} \label{sec:conclusion}
This paper studied tracking of planner-generated waypoint paths under speed and acceleration limits. Using the one-step acceleration margin from a reachability-guided tracker, we computed an offline time-scaling profile along a fitted spline to reduce one-step infeasibility while keeping the path geometry fixed. In randomized simulations with RRT* paths and a freeze--resume event, offline scaling reduced the fraction of moving samples with $\delta_k>0$ and lowered the mean and maximum values of $\delta$ while reducing the executed speed through local path-dependent slowdown.


The margin $\delta_k$ is a one-step screening signal under a constant-acceleration discrete model, so it does not provide a multi-step trackability guarantee. In addition, both the scaling profile and the feasibility statistics depend on the look-ahead policy through $s_{\text{LA},k}$ and the arc-length mapping in \eqref{eq:tau_la}. The current margin also uses a single symmetric bound $a_{\max}$ for both acceleration and deceleration. As a result, when the reference speed decreases along the path, $\delta_k$ can exhibit transient spikes because the controller must brake to match the slower reference. In such cases, the one-step test may mark a region as difficult even though the issue comes from the speed transition rather than from the path geometry alone. Finally, the current offline profile is implemented as piecewise-constant on a finite grid; enforcing a smooth timing law would require identifying a continuous function $\alpha(t)$ and accounting for the additional higher-order effects that arise from the dynamics.

Future work includes extending the one-step screening condition to multi-step feasibility, studying alternative look-ahead policies, designing smoother time-scaling profiles with rate constraints on $\alpha(\cdot)$, and validating the method beyond the planar double-integrator setting. It would also be useful to allow different acceleration and deceleration bounds to better match systems with asymmetric actuation limits.

\bibliographystyle{IEEEtran}
\bibliography{refs}

\end{document}